\documentclass[]{article}
\usepackage{proceed2e}
\usepackage{times}

\usepackage{graphicx} % more modern
\usepackage{subfigure} 
\usepackage{amssymb}
\usepackage{natbib}
\usepackage[colorlinks,citecolor=blue,urlcolor=blue,linkcolor=blue]{hyperref}
\usepackage{amsthm}
\usepackage{amsmath}
\usepackage{amssymb}
\usepackage{relsize}
\usepackage{graphicx}

\newcommand\independent{\protect\mathpalette{\protect\independenT}{\perp}}
\def\independenT#1#2{\mathrel{\rlap{$#1#2$}\mkern2mu{#1#2}}}
\newtheorem{definition}{Definition}
\newtheorem{lemma}{Lemma}
\newtheorem{theorem}{Theorem}
\newtheorem{claim}{Claim}

\newcommand{\widesim}[2][1.5]{
  \mathrel{\overset{#2}{\scalebox{#1}[1]{$\sim$}}}}
\usepackage{algorithm}
\usepackage{algorithmic}
\usepackage{hyperref}
\usepackage[T1]{fontenc}
\usepackage[utf8]{inputenc}

\newcommand*{\LargerCdot}{\raisebox{-0.25ex}{\scalebox{1.5}{$\cdot$}}}

\title{A Kernel Test for Three-Variable Interactions with Random Processes}

\author{ {\bf Paul K. Rubenstein$^{123}$}, \quad {\bf Kacper P. Chwialkowski$^3$}, \quad {\bf Arthur Gretton$^4$}  \\
$^1$Machine Learning Group, University of Cambridge\\
$^2$Empirical Inference, MPI for Intelligent Systems, Tübingen, Germany\\
$^3$Department of Computer Science, University College London\\
$^4$Gatsby Computational Neuroscience Unit, University College London\\
\texttt{pkr23@cam.ac.uk, kacper.chwialkowski@gmail.com, arthur.gretton@gmail.com}
}

\begin{document}
	\maketitle
\begin{abstract} 

We apply a wild bootstrap method to the Lancaster three-variable interaction measure in order to detect factorisation of the joint distribution on three variables forming a stationary random process, for which the existing permutation bootstrap method fails. As in the \emph{i.i.d.}~case, the Lancaster test is found to outperform existing tests in cases for which two independent variables individually have a weak influence on a third, but that when considered jointly the influence is strong. The main contributions of this paper are twofold: first, we prove that the Lancaster statistic satisfies the conditions required to estimate the quantiles of the null distribution using the wild bootstrap; second, the manner in which this is proved is novel, simpler than existing methods, and can further be applied to other statistics.

%(An additional minor contribution is that it is also shown that the multiple testing correction proposed in [Lancaster] is too conservative, and a new correction is proposed that increases test power)
\end{abstract} 

\section{INTRODUCTION}\label{section:intro}
\label{introduction}
Nonparametric testing of independence or interaction between random variables is a core staple of machine learning and statistics. The majority of nonparametric statistical tests of independence for continuous-valued random variables   rely on the assumption that the observed data are drawn \emph{i.i.d.} \cite{Feuerverger93,gretton2007kernel,Szekely2007,GreGyo10,HelHelGor13}. The same assumption applies to tests of conditional dependence, and of multivariate interaction between variables \cite{Zhang2011,KanUsh98,FukGreSunSch08,sejdinovic2013kernel,PatSenSze15}.
For many applications in finance, medicine, and audio signal analysis, however, the \emph{i.i.d.}~assumption is unrealistic and overly restrictive.
While many approaches exist for testing interactions between time series under strong parametric assumptions %\citep{shao_generalized_2009,fadlallah_association_2012}
\cite{kirchgassner2012introduction,ledford1996statistics}, the problem of testing for general, nonlinear interactions has seen far less analysis: tests of pairwise dependence have been proposed by \cite{GaiRupSch10,besserve_statistical_2013,chwialkowski2014wild, chwialkowski2014kernel}, where the first publication also addresses mutual independence of more than two univariate time series. The two final works use as their statistic the Hilbert-Schmidt Indepenence Criterion, a general nonparametric measure of dependence \citep{gretton2005measuring}, which applies even for multivariate or non-Euclidean variables (such as strings and groups).
The asymptotic behaviour and corresponding test threshold are derived using  particular assumptions on the mixing properties of the processes from which the observations are drawn. These kernel approaches apply only to pairs of random processes, however.

The Lancaster interaction is a signed measure that can be used to construct a test statistic capable of detecting dependence between three random variables \citep{lancaster1969chi,sejdinovic2013kernel}. If the joint distribution on the three variables factorises in some way into a product of a marginal and a pairwise marginal, the Lancaster interaction is zero everywhere. Given observations, this can be used to construct a statistical test, the null hypothesis of which is that the joint distribution factorises thus. 
In the \emph{i.i.d.}~case, the null distribution of the test statistic can be estimated using a permutation bootstrap technique: this amounts to shuffling the indices of one or more of the variables and recalculating the test statistic on this bootstrapped data set. When our samples instead exhibit temporal dependence, shuffling the time indices destroys this dependence and thus doing so does not correspond to a valid resample of the test statistic. 

Provided that our data-generating process satisfies some technical conditions on the forms of temporal dependence, recent work by \citet{leucht2013dependent}, building on the work of \citet{shao2010dependent}, can come to our rescue. The wild bootstrap is a method that correctly resamples from the null distribution of a test statistic, subject to certain conditions on both the test statistic and the processes from which the observations have been drawn.

In this paper we show that the Lancaster interaction test statistic satisfies the conditions required to apply the wild bootstrap procedure; moreover, the manner in which we prove this is significantly simpler than existing proofs in the literature of the same property for other kernel test statistics \citep{chwialkowski2014wild,chwialkowski2014kernel}. Previous proofs have relied on the classical theory of $V$-statistics to analyse the asymptotic distribution of the kernel statistic. In particular, the Hoeffding decomposition gives an expression for the kernel test statistic as a sum of other $V$-statistics. Understanding the asymptotic properties of the components of this decomposition is then conceptually tractable, but algebraically extremely painful. Moreover, as the complexity of the test statistic under analysis grows, the number of terms that must be considered in this approach grows factorially.\footnote{See for example Lemma 8 in Supplementary material A.3 of \citet{chwialkowski2014kernel}. The proof of this lemma requires keeping track of $4!$ terms; an equivalent approach for the Lancaster test would have $6!$ terms. Depending on the precise structure of the statistic, this approach applied to a test involving 4 variables could require as many as $8!=40320$ terms.} We conjecture that such analysis of interaction statistics of 4 or more variables would in practice be unfeasible without automatic theorem provers due to the sheer number of terms in the resulting computations.

In contrast, in the approach taken in this paper we explicitly consider our test statistic to be the norm of a Hilbert space operator. We exploit a Central Limit Theorem for Hilbert space valued random variables \cite{dehling2015bootstrap} to show that our test statistic converges in probability to the norm of a related population-centred Hilbert space operator, for which the asymptotic analysis is much simpler. Our approach is novel; previous analyses have not, to our knowledge, leveraged the Hilbert space geometry in the context of statistical hypothesis testing using kernel $V$-statistics in this way. 

We propose that our method may in future be applied to the asymptotic analysis of other kernel statistics. In the appendix, we provide an application of this method to the Hilbert Schmidt Independence Criterion (HSIC) test statistic, giving a significantly shorter and simpler proof than that given in \citet{chwialkowski2014kernel}

The Central Limit Theorem that we use in this paper makes certain assumptions on the mixing properties of the random processes from which our data are drawn; as further progress is made, this may be substituted for more up-to-date theorems that make weaker mixing assumptions.

\paragraph{OUTLINE:}
In Section \ref{section:main}, we detail the Lancaster interaction test and provide our main results. These results justify use of the wild bootstrap to understand the null distribution of the test statistic. In Section \ref{section:details}, we provide more detail about the wild bootstrap, prove that its use correctly controls Type I error and give a consistency result. In Section \ref{section:experiments}, we evaluate the Lancaster test on synthetic data to identify cases in which it outperforms existing methods, as well as cases in which it is outperformed. In Section \ref{section:proofs}, we provide proofs of the main results of this paper, in particular the aforementioned novel proof. Further proofs may be found in the Supplementary material.

\section{LANCASTER INTERACTION TEST}\label{section:main}

\subsection{KERNEL NOTATION}

Throughout this paper we  will assume that the kernels $k,l,m$, defined on the domains $\mathcal{X}$, $\mathcal{Y}$ and $\mathcal{Z}$ respectively, are characteristic  \citep{sriperumbudur2011universality}, bounded and Lipschitz continuous. We describe some notation relevant to the kernel $k$; similar notation holds for $l$ and $m$. Recall that $\mu_X := \mathbb{E}_X k(X,\cdot) \in \mathcal{F}_k$ is the  mean embedding \citep{smola2007hilbert} of the random variable $X$. Given observations $X_i$, an estimate of the mean embedding is $\tilde{\mu}_X = \frac{1}{n}\sum_{i=1}^n k(X_i,\cdot)$. Two modifications  of $k$ are used in this work:
\begin{align}
\bar{k}(x,x') &= \langle  k(x,\cdot)-\mu_X,k(x',\cdot)-\mu_X\rangle, \\
\tilde{k}(x,x') &= \langle k(x,\cdot)-\tilde{\mu}_X, k(x',\cdot)-\tilde{\mu}_X \rangle 
\end{align}
These are called the \emph{population centered kernel} and \emph{empirically centered kernel} respectively. 

\subsection{LANCASTER INTERACTION}

The Lancaster interaction on the triple of random variables $(X,Y,Z)$ is defined as the signed measure $\Delta_LP = \mathbb{P}_{XYZ} - \mathbb{P}_{XY}\mathbb{P}_{Z} - \mathbb{P}_{XZ}\mathbb{P}_{Y} - \mathbb{P}_{X}\mathbb{P}_{YZ} + 2\mathbb{P}_{X}\mathbb{P}_{Y}\mathbb{P}_{Z}$. This measure can be used to detect three-variable interactions. It is straightforward to show that if any variable is independent of the other two (equivalently, if the joint distribution $\mathbb{P}_{XYZ}$ factorises into a product of marginals in any way), then $\Delta_LP = 0$. That is, writing $\mathcal{H}_X = \{X \independent (Y,Z)\}$ and similar for $\mathcal{H}_Y$ and $\mathcal{H}_Z$, we have that

\begin{equation}\label{eqn:lancaster-zero}
\mathcal{H}_X \enspace \lor \enspace \mathcal{H}_Y \enspace \lor \enspace \mathcal{H}_Z \Rightarrow \Delta_LP=0
\end{equation}

The reverse implication does not hold, and thus no conclusion about the veracity of the $\mathcal{H}_{\LargerCdot}$ can be drawn when $\Delta_LP=0$. Following \citet{sejdinovic2013kernel}, we can consider the mean embedding of this measure:
\begin{align}
 \mu_L = \int k(x,\cdot) l(y,\cdot) m(z,\cdot) \Delta_LP 
\end{align}

 Given an \emph{i.i.d.}~sample $(X_i,Y_i,Z_i)_{i=1}^n$, the norm of the mean embedding $\mu_L$ can be empirically estimated using empirically centered kernel matrices. For example, for the kernel $k$ with kernel matrix $K_{ij} = k(X_i,X_j)$, the empirically centered kernel matrix $\tilde{K}$ is given by
\[
\tilde{K}_{ij} = \langle k(X_i,\cdot)-\tilde{\mu}_X, k(X_j,\cdot) -\tilde{\mu}_X \rangle,
\]
 By \citet{sejdinovic2013kernel}, an estimator of the norm of the mean embedding of the Lancaster interaction for \emph{i.i.d.}~samples is 
\begin{equation}\label{eqn:lancaster}
\|\hat \mu_L\|^2 = \frac{1}{n^2}\left(\tilde{K}\circ\tilde{L}\circ\tilde{M}\right)_{++}
\end{equation}
where $\circ$ is the Hadamard (element-wise) product and $A_{++} = \sum_{ij}A_{ij}$, for a matrix $A$.

\subsection{TESTING PROCEDURE}

In this paper, we construct a statistical test for three-variable interaction, using $n\|\hat \mu_L\|^2$ as the test statistic to distinguish between the following hypotheses:

$\mathcal{H}_0: \mathcal{H}_X \enspace \lor \enspace \mathcal{H}_Y \enspace \lor \enspace \mathcal{H}_Z $\\
$\mathcal{H}_1: \mathbb{P}_{XYZ}$ does not factorise in any way

The null hypothesis $\mathcal{H}_0$ is a composite of the three `sub-hypotheses' $\mathcal{H}_X$, $\mathcal{H}_Y$ and $\mathcal{H}_Z$. We test $\mathcal{H}_0$ by testing each of the sub-hypotheses separately and we reject if and only if we reject each of $\mathcal{H}_X$, $\mathcal{H}_Y$ and $\mathcal{H}_Z$. Hereafter we describe the procedure for testing $\mathcal{H}_Z$; similar results hold for $\mathcal{H}_X$ and $\mathcal{H}_Y$.

\citet{sejdinovic2013kernel} show that, under $\mathcal{H}_Z$, $n \|\hat \mu_L\|^2 $ converges to an infinite sum of weighted $\chi$-squared random variables. By leveraging the \emph{i.i.d.}~assumption of the samples, any given quantile of this  distribution  can be estimated using simple permutation bootstrap, and so a test procedure is proposed.

In the time series setting this approach does not work. Temporal dependence within the samples makes  study of the asymptotic distribution of $n \|\hat \mu_L\|^2 $ difficult; in Section \ref{experiment2} we verify experimentally that the permutation bootstrap used in the \emph{i.i.d}~case fails. To construct a test in this setting we will use asymptotic and bootstrap results for mixing processes. 

Mixing formalises the notion of the temporal structure within a process, and can be thought of as the rate at which the process forgets about its past. For example, for Gaussian processes this rate can be captured by the autocorrelation function; for general processes, generalisations of autocorrelation are used. The exact assumptions we make about the mixing properties of processes in this paper are discussed in Section \ref{section:details}, and we will refer to them as \textit{suitable mixing assumptions} for brevity in statements of results throughout this paper.

\subsection{MAIN RESULTS}

It is straightforward to show that the norm of the mean embedding \eqref{eqn:lancaster} can also be written as

\[ \|\hat \mu_L\|^2 = \frac{1}{n^2}\left(\widetilde{\tilde{K}\circ\tilde{L}}\circ\tilde{M}\right)_{++}\]

Our first contribution is to show that the (difficult) study of the asymptotic null distribution of $ \|\hat \mu_L\|^2$ can be reduced to studying population centered kernels
\[
\| \hat \mu^{(Z)}_{L,2} \|^2 =\frac{1}{n^2}\left(\overline{\overline{K}\circ\overline{L}}\circ\overline{M}\right)_{++}
\]
where e.g. 
\[
\overline{K}_{ij} = \langle k(X_i,\cdot)-\mu_X, k(X_j,\cdot) -\mu_X \rangle,
\]
Specifically, we prove the following:
\begin{theorem}\label{theorem:norm-conv-in-prob} Suppose that $(X_i,Y_i,Z_i)_{i=1}^n$ are drawn from a random process satisfying suitable mixing assumptions. Under $\mathcal{H}_Z$, $\lim_{n \to \infty} ( n\| \hat \mu^{(Z)}_{L,2} \|^2 - n\|\hat \mu_L\|^2 ) =0 $ in probability.
\end{theorem}

Our proof of Theorem \ref{theorem:norm-conv-in-prob} relies crucially on the following Lemma which we prove in Supplementary material \ref{supp:hilbert-clt}
\begin{lemma}\label{lemma:hilbertCLT}
Suppose that $(X_i)_{i=1}^n$ is drawn from a random process satisfying suitable mixing assumptions and that $k$ is a bounded kernel on $\mathcal{X}$. Then $\|\hat\mu_X - \mu_X\|_k = O_P(n^{-\frac{1}{2}})$

\end{lemma}

\begin{proof}\textit{(Theorem \ref{theorem:norm-conv-in-prob})}
We provide a short sketch of the proof here; for a full proof, see Section \ref{section:proofs}.

The key idea is to note that we can rewrite $n\|\hat \mu_L\|^2$ in terms of the population centred kernel matrices $\overline{K}$, $\overline{L}$ and $\overline{M}$. Each of the resulting terms can in turn be converted to an inner product between quantities of the form $\hat\mu - \mu$, where $\hat\mu$ is an empirical estimator of $\mu$, and each $\mu$ is a mean embedding or covariance operator.

By applying Lemma \ref{lemma:hilbertCLT} to the $\hat\mu - \mu$, we show that most of these terms converge in probability to 0, with the residual terms equaling $n\| \hat \mu^{(Z)}_{L,2} \|^2$.
\end{proof}
As discussed in Section \ref{section:intro}, the essential idea of this proof is novel and the resulting proof is significantly more concise than previous approaches \citep{chwialkowski2014kernel,chwialkowski2014wild}. 

Theorem \ref{theorem:norm-conv-in-prob} is useful because the statistic $\| \hat \mu^{(Z)}_{L,2} \|^2$ is much easier to study under the non-\emph{i.i.d.}~assumption than $\|\hat \mu_L\|^2$. Indeed, it can expressed as a $V$-statistic (see Section \ref{subsection:v-statistic})
\[ 
V_n = \frac{1}{n^2} \mathlarger{\sum}_{1\leq i,j \leq n} \overline{\overline{k} \otimes \overline{l}}\otimes \overline{m} (S_i,S_j)
\]
where  $S_i = (X_i,Y_i,Z_i)$. The crucial observation is that
\[
 h := \overline{\overline{k} \otimes \overline{l}}\otimes \overline{m}
\]
is well behaved in the following sense.
\begin{theorem}\label{theorem:degenerate-kernel}
Suppose that $k$, $l$ and $m$ are bounded, symmetric, Lipschitz continuous kernels. Then $h$ is also bounded symmetric and Lipschitz continuous, and is moreover degenerate under $\mathcal{H}_Z$ i.e $\mathbb{E}_{S}h(S,s)=0$ for any fixed $s$.
\end{theorem}
\begin{proof}
See Section \ref{section:proofs}
\end{proof}
The asymptotic analysis of such a $V$-statistic for non-\emph{i.i.d.}~data is still complex, but we can appeal to prior work: \citet{leucht2013dependent} showed a way to estimate any given quantile of such a $V$-statistic under the null hypothesis using a method called the wild bootstrap. This, combined with analysis of the $V$-statistic under the alternative hypothesis provided in Theorem 2 of \citet{chwialkowski2014wild}\footnote{Note that similar results are presented in \citet{leucht2013dependent} as specific cases.}, results in statistical test (see Algorithm \ref{alg:Lancaster}).
\begin{algorithm}[tb]
   \caption{Test $\mathcal{H}_Z$ with Wild Bootstrap}
   \label{alg:Lancaster}
\begin{algorithmic}
   \STATE {\bfseries Input:} $\tilde{K}$, $\tilde{L}$, $\tilde{M}$, each size $n\times n$, $N$= number of bootstraps, $\alpha=$ p-value threshold
   \STATE $n\|\hat{\mu}_L\|^2 = \frac{1}{n}\left(\widetilde{\left( \tilde{K} \circ \tilde{L}\right) }\circ \tilde{M} \right)_{++}$
   \STATE samples = zeros(1,N)
   \FOR{$i=1$ {\bfseries to} $N$}
   \STATE Draw random vector W according to Equation \ref{equation:bootstrap}
   \STATE samples[$i$] = $\frac{1}{n}W^\intercal\left( \widetilde{\left( \tilde{K} \circ \tilde{L}\right) }\circ \tilde{M} \right)W$
   \ENDFOR
   \IF{sum($n\|\hat{\mu}_L\|^2 >$ samples)$>\frac{\alpha}{N}$}
   \STATE Reject $\mathcal{H}_Z$
   \ELSE
   \STATE Do not reject $\mathcal{H}_Z$
   \ENDIF
\end{algorithmic}
\end{algorithm}

In Section \ref{section:details} we discuss the wild bootstrap and provide results regarding consistency and Type I error control.

\subsection{MULTIPLE TESTING CORRECTION}
In the Lancaster test, we reject the composite null hypothesis $\mathcal{H}_0$ if and only if we reject all three of the components. In \citet{sejdinovic2013kernel}, it is suggested that the Holm-Bonferroni correction be used to account for multiple testing \citep{holm1979simple}. We show here that more relaxed conditions on the p-values can be used while still bounding the Type I error, thus increasing test power.

Denote by $\mathcal{A}_*$ the event that $\mathcal{H}_*$ is rejected. Then
\begin{align*}
\mathbb{P}(\mathcal{A}_0) &= \mathbb{P}(\mathcal{A}_X \land \mathcal{A}_Y \land \mathcal{A}_Z) \\
&\leq \min\{\mathbb{P}(\mathcal{A}_X), \mathbb{P}(\mathcal{A}_Y), \mathbb{P}(\mathcal{A}_Z)\}
\end{align*}
If $\mathcal{H}_0$ is true, then so must one of the components. Without loss of generality assume that $\mathcal{H}_X$ is true. If we use significance levels of $\alpha$ in each test individually then $\mathbb{P}(\mathcal{A}_X) \leq \alpha$ and thus $\mathbb{P}(\mathcal{A}_0) \leq \alpha$.

Therefore rejecting $\mathcal{H}_0$ in the event that each test has p-value less than $\alpha$ individually guarantees a Type I error overall of at most $\alpha$. In contrast, the Holm-Bonferonni method requires that the sorted p-values be lower than $[\frac{\alpha}{3},\frac{\alpha}{2},\alpha]$ in order to reject the null hypothesis overall. It is therefore more conservative than necessary and thus has worse test power compared to the `simple correction' proposed here. This is experimentally verified in Section \ref{section:experiments}.

\section{THE WILD BOOTSTRAP}\label{section:details}

In this section we discuss the wild bootstrap and provide consistency and Type I error results for the proposed Lancaster test.

\subsection{TEMPORAL DEPENDENCE}
There are various formalisations of memory or `mixing' of a random process \citep{doukhan1994mixing,bradley2005basic,dedecker2007weak}; of relevance to this paper is the following :

% \begin{definition}\cite{leucht2013dependent}
% A process $(X_t)_{t}$ is \emph{$\tau$-mixing} if $\tau(r) \longrightarrow 0$ as $r\longrightarrow \infty$, where
% \[\tau(r) = \sup_{l\in \mathbb{N}} \frac{1}{l} \sup_{r\leq i_1 \leq \ldots \leq i_l} \tau(\mathcal{F}_0,(X_{i_1}, \ldots, X_{i_l})) \longrightarrow 0\]
% where
% \[ \tau(\mathcal{M},X) = \mathbb{E} (\sup_{g \in \Lambda} | \int g(t) \mathbb{P}_{X|\mathcal{M}}(dt) -  \int g(t) \mathbb{P}_{X}(dt) |)\]
% \end{definition}

\begin{definition}
A process $(X_t)_{t}$ is \emph{$\beta$-mixing} (also known as \emph{absolutely regular}) if $\beta(m) \longrightarrow 0$ as $m\longrightarrow \infty$, where
\[ \beta(m) = \frac{1}{2} \sup_n \sup \sum_{i=1}^I \sum_{j=1}^J | \mathbb{P}(A_i \cap B_j) - \mathbb{P}(A_i)\mathbb{P}(B_j)| \]
where the second supremum is taken  over all finite partitions $\{A_1,\ldots, A_I \}$ and  $\{B_1,\ldots, B_J\}$ of the sample space such that $A_i \in \mathcal{F}_1^n$ and $B_j \in \mathcal{F}_{n+m}^\infty$ and $\mathcal{F}_b^c = \sigma(X_b,X_{b+1},\ldots,X_{c})$
\end{definition}
A related notion is that of $\tau$-mixing. This is a property required to apply the wild bootstrap method of \citet{leucht2013dependent}, but we do not discuss $\tau$-mixing here since it is implied by $\beta$-mixing under the assumption that $X_i$ has finite $p$-th moment for any $p>1$.

\subsection*{SUITABLE MIXING ASSUMPTIONS} We assume that the random process $S_i = (X_i,Y_i,Z_i)$ is $\beta$ mixing with mixing coefficients satisfying   $\beta(m)  = o(m^{-6})$. Throughout this paper we refer to this assumption as \emph{suitable mixing assumptions}.

\subsection{$V$-STATISTICS}\label{subsection:v-statistic}

A $V$-statistic of a 2-argument, symmetric function $h$ given  observations $\mathcal{S}_n = \{S_1,\ldots,S_n\}$  is \citep{serfling2009approximation}:

\[ V_n =  \frac{1}{n^2} \mathlarger{\sum}_{1\leq i,j \leq n} h(S_i,S_j)\]

We call $nV_n$ a \emph{normalised} $V$-statistic. We call $h$ the \emph{core} of $V$ and we say that $h$ is \emph{degenerate} if, for any $s_1$, $\mathbb{E}_{S_2 \sim \mathbb{P}}[h(s_1,S_2)] = 0$, in which case we say that $V$ is a \emph{degenerate $V$-statistic}. Many kernel test statistics can be viewed as normalised $V$-statistics which, under the null hypothesis, are degenerate. As mentioned in the previous section,  $\|\hat \mu^{(Z)}_{L,2}\|^2$ is a $V$-statistic. Theorems \ref{theorem:norm-conv-in-prob} and \ref{theorem:degenerate-kernel} together imply that, under $\mathcal{H}_Z$, it can be treated as a degenerate $V$-statistic. 

\subsection{WILD BOOTSTRAP}

If the test statistic has the form of a normalised $V$-statistic, then provided certain extra conditions are met, the wild bootstrap of \citet{leucht2013dependent} is a method to directly resample the test statistic under the null hypothesis. These conditions can be categorised as concerning: (1) appropriate mixing of the process from which our observations are drawn; (2) the core of the $V$-statistic. 

The condition on the core that is of crucial importance to this paper is that it must be degenerate. Theorem \ref{theorem:degenerate-kernel} justifies our use of the wild bootstrap in the Lancaster interaction test.

Given the statistic $nV_n$, \citet{leucht2013dependent} tells us that a random vector $W$ of length $n$ can be drawn such that the bootstrapped statistic\footnote{Note that for fixed $\mathcal{S}_n$, $nV_b$ is a random variable through the randomness introduced by $W$}
\[nV_b=\frac{1}{n}\sum_{i,j}W_{i}h(S_i,S_j)W_{j}\]
is distributed according to the null distribution of $nV_n$. 

By generating many such $W$ and calculating $nV_b$ for each, we can estimate the quantiles of $nV$. 

\subsection{GENERATING $W$}

The process generating $W$ must satisfy conditions (B2) given on page 6 of \citet{leucht2013dependent} for $nV_b$ to correctly resample from the null distribution of $nV_n$. For brevity, we provide here only an example of such a process; the interested reader should consult \citet{leucht2013dependent} or Appedix A of \citet{chwialkowski2014wild} for a more detailed discussion of the bootstrapping process. The following bootstrapping process was used in the experiments in Section \ref{section:experiments}:
\begin{equation}\label{equation:bootstrap}
W_t = e^{-1/l_n}W_{t-1} + \sqrt{1 - e^{-2/l_n}}\epsilon_t 
\end{equation}
where $W_1$, $\epsilon_1, \ldots, \epsilon_t$ are independent $\mathcal{N}(0,1)$ random variables. $l_n$ should be taken from a sequence $\{l_n\}$ such that $\lim_{n\longrightarrow\infty}l_n = \infty$; in practice we used $l_n=20$ for all of the experiments since the values of $n$ were roughly comparable in each case.
\subsection{CONTROL OF TYPE I ERROR}

The following theorem shows that by estimating the quantiles of the wild bootstrapped statistic $nV_b$ we correctly control the Type I error when testing $\mathcal{H}_Z$.

\begin{theorem}\label{theorem:quantiles-converge}
Suppose that $(X_i,Y_i,Z_i)_{i=1}^n$ are drawn from a random process satisfying suitable mixing conditions, and that $W$ is drawn from a process satisfying (B2) in \citet{leucht2013dependent}. Then asymptotically, the quantiles of
\[nV_b = \frac{1}{n}W^\intercal\left( \overline{\left( \bar{K} \circ \bar{L}\right) }\circ \bar{M} \right)W\]
converge to those of $ n\| \hat \mu_L\|^2$. 
\end{theorem}

\begin{proof}
See Supplementary material \ref{supp:quantile-proof}
\end{proof}

%TODO proof!

\subsection{(SEMI-)CONSISTENCY OF TESTING PROCEDURE}

Note that in order to achieve consistency for this test, we would need that $\mathcal{H}_0 \iff \Delta_LP = 0$. Unfortunately this does not hold - in \citet{sejdinovic2013kernel} examples are given of distributions for which $\mathcal{H}_0$ is false, and yet $\Delta_LP = 0$. 

However, the following result does hold:

\begin{theorem}\label{theorem:consistent}
Suppose that $\Delta_LP \not =0$. Then as $n\longrightarrow\infty$, the probability of correctly rejecting $\mathcal{H}_0$ converges to 1.
\end{theorem}
\begin{proof}
See Supplementary material \ref{supp:consistent}
\end{proof}

At the time of writing, a characterisation of distributions for which $\mathcal{H}_0$ is false yet $\Delta_LP=0$ is unknown. Therefore, if we reject $\mathcal{H}_0$ then we conclude that the distribution does not factorise; if we fail to reject $\mathcal{H}_0$ then we cannot conclude that the distribution factorises.

\section{EXPERIMENTS}\label{section:experiments}

The Lancaster test described above amounts to a method to test each of the sub-hypotheses $\mathcal{H}_X, \mathcal{H}_Y, \mathcal{H}_Z$. Rather than using the Lancaster test statistic with wild bootstrap to test each of these, we could instead use HSIC. For example, by considering the pair of variables $(X,Y)$ and $Z$ with kernels $k\otimes l$ and $m$ respectively, HSIC can be used to test $\mathcal{H}_Z$. Similar grouping of the variables can be used to test $\mathcal{H}_X$ and $\mathcal{H}_Y$. Applying the same multiple testing correction as in the Lancaster test, we derive an alternative test of dependence between three variables. We refer to this HSIC based procedure as \emph{3-way HSIC}.

In the case of \emph{i.i.d.}~observations, it was shown in \citet{sejdinovic2013kernel} that Lancaster statistical test is more sensitive to dependence between three random variables than the above HSIC-based test when pairwise interaction is weak but joint interaction is strong. In this section, we demonstrate that the same is true in the time series case on synthetic data.

\subsection{WEAK PAIRWISE INTERACTION, STRONG JOINT INTERACTION}\label{experiment1}

This experiment demonstrates that the Lancaster test has greater power than 3-way HSIC when the pairwise interaction is weak, but joint interaction is strong.

%Example 2 in thesis. 3-way HSIC in principle should be able to detect the interaction, but Lancaster is much more powerful. See Figure \ref{weak-pairwise-strong-joint}.

Synthetic data were generated from autoregressive processes $X$, $Y$ and $Z$ according to:
\begin{align*}
X_t &= \frac{1}{2}X_{t-1} + \epsilon_t\\
Y_t &= \frac{1}{2}Y_{t-1} + \eta_t\\
Z_t &= \frac{1}{2}Z_{t-1} + d |\theta_t|\text{sign}(X_t Y_t) + \zeta_t\\
\end{align*}
where $X_0, Y_0, Z_0, \epsilon_t, \eta_t, \theta_t$ and $\zeta_t$ are \emph{i.i.d.}~$\mathcal{N}(0,1)$ random variables and $d\in\mathbb{R}$, called the \emph{dependence} coefficient, determines the extent to which the process $(Z_t)_t$ is dependent on $(X_t,Y_t)_t$.

Data were generated with varying values of $d$. For each value of $d$, 300 datasets were generated, each consisting of 1200 consecutive observations of the variables. Gaussian kernels with bandwidth parameter 1 were used on each variable, and 250 bootstrapping procedures were used for each test on each dataset.

Observe that the random variables are pairwise independent but jointly dependent. Both the Lancaster and 3-way HSIC tests should be able to detect the dependence and therefore reject the null hypothesis in the limit of infinite data. In the finite data regime, the value of $d$ affects drastically how hard it is to detect the dependence. The results of this experiment are presented in Figure \ref{weak-pairwise-strong-joint}, which shows that the Lancaster test achieves very high test power with weak dependence coefficients compared to 3-way HSIC. Note also that when using the simple multiple testing correction a higher test power is achieved than with the Holm-Bonferroni correction.

\begin{figure}[ht]
\vskip 0.2in
\begin{center}
\centerline{\includegraphics[scale=0.6]{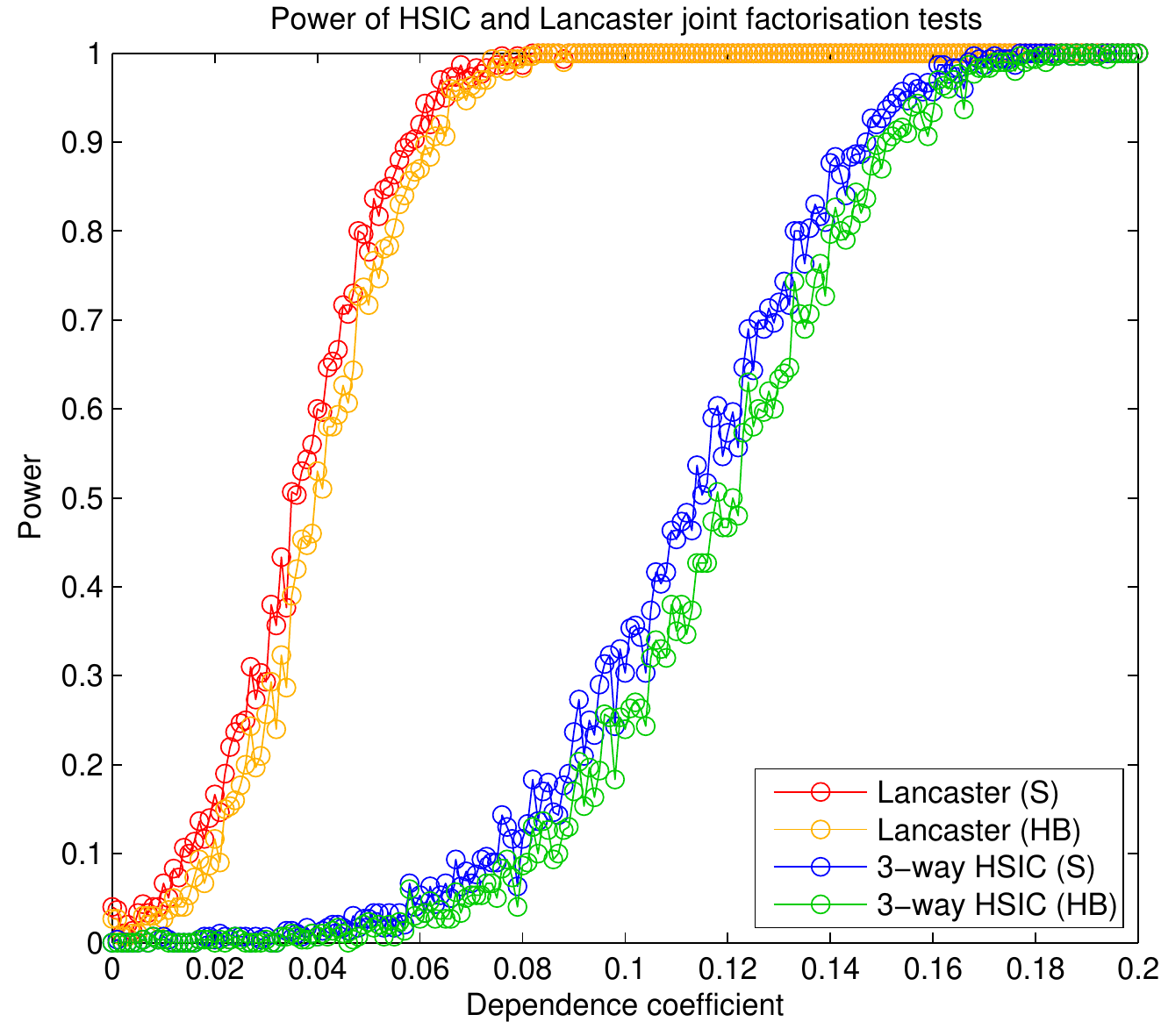}}
\caption{Results of experiment in Section \ref{experiment1}. (S) refers to the simple multiple correction; (HB) refers to Holm-Bonferroni. The Lancaster test is more sensitive to dependence than 3-way HSIC, and test power for both tests is higher when using the simple correction rather than the Holm-Bonferroni multiple testing correction.}
\label{weak-pairwise-strong-joint}
\end{center}
\vskip -0.2in
\end{figure} 

\subsection{FALSE POSITIVE RATES}\label{experiment2}
%Example 4 in thesis. Comparison of wild bootstrap to permutation bootstrap

This experiment demonstrates that in the time series case, existing permutation bootstrap methods fail to control the Type I error, while the  wild bootstrap correctly identifies test statistic thresholds and appropriately controls Type I error.

Synthetic data were generated from autoregressive processes $X$, $Y$ and $Z$ according to:
\begin{align*}
X_t &= aX_{t-1} + \epsilon_t\\
Y_t &= aY_{t-1} + \eta_t\\
Z_t &= aZ_{t-1} +  \zeta_t\\
\end{align*}
where $X_0, Y_0, Z_0, \epsilon_t, \eta_t$ and $\zeta_t$ are \emph{i.i.d.}~$\mathcal{N}(0,1)$ random variables and $a$, called the \emph{dependence coefficient}, determines how temporally dependent the processes are. The null hypothesis in this example is true as each process is independent of the others.

The Lancaster test was performed using both the Wild Bootstrap and the simple permutation bootstrap (used in the \emph{i.i.d.}~case) in order to sample from the null distributions of the test statistic. We used a fixed desired false positive rate $\alpha = 0.05$ with sample of size 1000, with 200 experiments run for each value of $a$. Figure \ref{wildBootstrap_is_necessary} shows the false positive rates for these two methods for varying $a$. It shows that as the processes become more dependent, the false positive rate for the permutation method becomes very large, and is not bounded by the fixed $\alpha$, whereas the false positive rate for the Wild Bootstrap method is bounded by $\alpha$.
\begin{figure}[ht]
\vskip 0.2in
\begin{center}
\centerline{\includegraphics[scale=0.6]{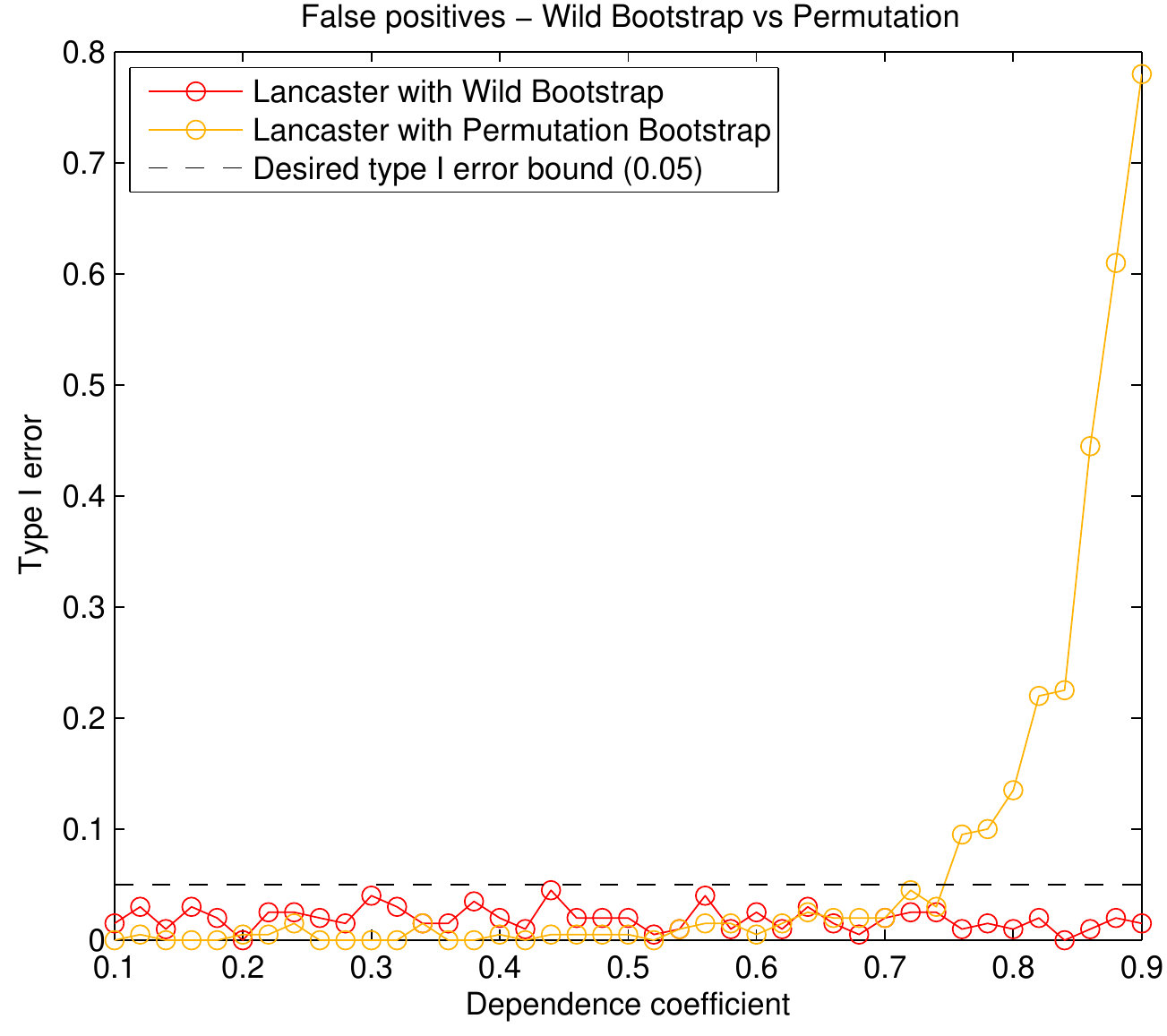}}
\caption{Results of experiment in section \ref{experiment2}. Whereas the wild bootstrap succeeds in controlling the Type I error across all values of the dependence coefficient, the permutation bootstrap fails to control the Type I error as it does not sample from the correct null distribution as temporal dependence between samples increases.}
\label{wildBootstrap_is_necessary}
\end{center}
\vskip -0.2in
\end{figure}

\subsection{STRONG PAIRWISE INTERACTION}\label{experiment3}
This experiment demonstrates a limitation of the Lancaster test. When pairwise interaction is strong, 3-way HSIC has greater test power than Lancaster.

Synthetic data were generated from autoregressive processes $X$, $Y$ and $Z$ according to:
\begin{align*}
X_t &= \frac{1}{2}X_{t-1} + \epsilon_t\\
Y_t &= \frac{1}{2}Y_{t-1} + \eta_t\\
Z_t &= \frac{1}{2}Z_{t-1} + d(X_t + Y_t) + \zeta_t\\
\end{align*}
where $X_0, Y_0, Z_0, \epsilon_t, \eta_t$ and $\zeta_t$ are \emph{i.i.d.}~$\mathcal{N}(0,1)$ random variables and $d\in\mathbb{R}$, called the \emph{dependence} coefficient, determines the extent to which the process $(Z_t)_t$ is dependent on $X_t$ and $Y_t$.

Data were generated with varying values for the dependence coefficient. For each value of $d$, 300 datasets were generated, each consisting of 1200 consecutive observations of the variables. Gaussian kernels with bandwidth parameter 1 were used on each variable, and 250 bootstrapping procedures were used for each test on each dataset.

In this case $Z_t$ is pairwise-dependent on both of $X_t$ and $Y_t$, in addition to all three variables being jointly dependent. Both the Lancaster and 3-way HSIC tests should be capable of detecting the dependence and therefore reject the null hypothesis in the limit of infinite data. The results of this experiment are presented in Figure \ref{strong-pairwise}, which demonstrates that in this case the 3-way HSIC test is more sensitive to the dependence than the Lancaster test.

\begin{figure}[ht]
\vskip 0.2in
\begin{center}
\centerline{\includegraphics[scale=0.6]{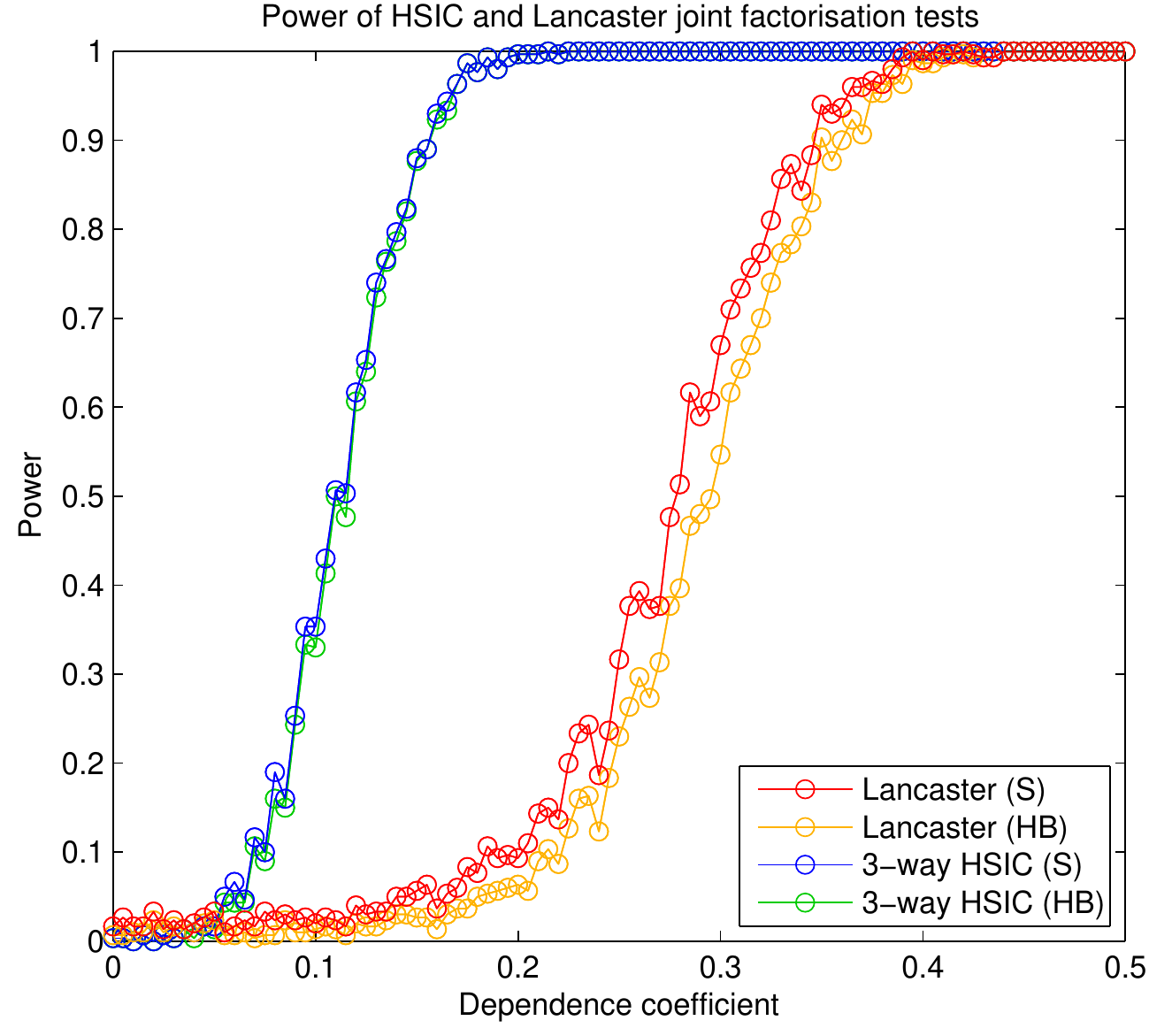}}
\caption{Results of experiment in Section \ref{experiment3}. (S) refers to the simple multiple correction; (HB) refers to Holm-Bonferroni. The Lancaster test is less sensitive to dependence than 3-way HSIC, and test power in both cases is higher when using the simple correction rather than the Holm-Bonferroni multiple testing correction.}
\label{strong-pairwise}
\end{center}
\vskip -0.2in
\end{figure} 

\subsection{FOREX DATA}

Exchange rates between three currencies (GBP, USD, EUR) at 5 minute intervals over 7 consecutive trading days were obtained. The data were processed by taking the returns (difference between consecutive terms within each time series, $x_t^r = x_t-x_{t-1}$) which were then normalised (divided by standard deviation). We performed the Lancaster test, 3-way HSIC and pairwise HSIC on using the first $800$ entries of each processed series. All tests rejected the null hypothesis. The Lancaster test returned $p$-values of 0 for each of $\mathcal{H}_X$, $\mathcal{H}_Y$ and $\mathcal{H}_Z$ with $10000$ bootstrapping procedures.

We then shifted one of the time series and repeated the tests (i.e.~we used entries $1$ to $800$ of two of the processed series and entries $801$ to $1600$ of the third). In this case, pairwise HSIC still detected dependence between the two unshifted time series, and both Lancaster and 3-way HSIC did not reject the null hypothesis that the joint distribution factorises. The Lancaster test returned $p$-values of $0.2708$, $0.2725$ and $0.1975$ for $\mathcal{H}_X$, $\mathcal{H}_Y$ and $\mathcal{H}_Z$ respectively.

In both cases, the Lancaster test behaves as expected. Due to arbitrage, any two exchange rates should determine the third and the Lancaster test correctly identifies a joint dependence in the returns. However, when we shift one of the time series, we break the dependence between it and the other series. Lancaster correctly identifies here that the underlying distribution does factorise.

\section{DISCUSSION AND FUTURE RESEARCH}\label{section:discussion}

We demonstrated that the Lancaster test is more sensitive than 3-way HSIC when pairwise interaction is weak, but that the opposite is true when pairwise interaction is strong. It is curious that the two tests have different strengths in this manner, particularly when considering the very similar forms of the statistics in each case. Indeed, to test $\mathcal{H}_Z$ using the Lancaster statistic, we bootstrap the following:

\begin{align*}
n\|\Delta_L\hat{P}\|^2 = \frac{1}{n}\left(\widetilde{\left( \tilde{K} \circ \tilde{L}\right) }\circ \tilde{M} \right)_{++}
\end{align*}

while for the 3-way HSIC test we bootstrap:

\begin{align*}
nHSIC_b = \frac{1}{n}\left(\widetilde{\left( K \circ L\right) }\circ \tilde{M} \right)_{++}
\end{align*}

These two quantities differ only in the centring of $K$ and $L$, amounting to constant shifts in the respective feature spaces of the kernels $k$ and $l$. This difference has the consequence of quite drastically changing the types of dependency to which each statistic is sensitive. A formal characterisation of the cases in which the Lancaster statistic is more sensitive than 3-way HSIC would be desirable.

\section{PROOFS}\label{section:proofs}

An outline of the proof of Theorem \ref{theorem:norm-conv-in-prob} was given in Section \ref{section:main}; here we provide the full proof, as well as a proof of Theorem \ref{theorem:degenerate-kernel}. 
%
%In this section we present proofs of Theorems \ref{theorem:norm-conv-in-prob} and \ref{theorem:degenerate-kernel}. The essential idea of the proof presented in this paper is to rewrite the test statistic as a sum of terms involving population centred gram matrices. By exploiting a central limit theorem for Hilbert space valued random variables (Lemma \ref{lemma:hilbertCLT} in Supplementary material), we show that under the null hypothesis all but one of these terms decay to $0$ as $n \longrightarrow \infty$. 
%
%In contrast, existing proof methods have employed the theory of U- and $V$-statistics \citep{chwialkowski2014kernel,chwialkowski2014wild}; in particular, the Hoeffding decomposition of the core of a $V$-statistic as a sum of other cores \citep{serfling2009approximation}. This allows the rewriting of the $V$-statistic as a sum of other $V$-statistics, which under the null hypothesis decay to 0.
%
%Both approaches amount to the same result, but they tackle the issue of centring of kernels in feature space in different ways. By appealing to a central limit theorem, the kernels are centred directly in the proof presented here. In contrast, the centring is obscured behind layers of algebra and theory in the previously presented proofs.
%
%The approach taken in this paper can be modified to give simpler a proof of similar theorems corresponding to HSIC than that given in \citet{chwialkowski2014wild}; this is provided in the appendix.

\begin{proof}(Theorem \ref{theorem:norm-conv-in-prob})

By observing that
\begin{align*}
& \phi_X(X_i)- \frac{1}{n}\sum_k\phi_X(X_k) \\
= \enspace&  (\phi_X(X_i) - \mu_X) - \frac{1}{n}\sum_k (\phi_X(X_k) - \mu_X)\\
= \enspace&\bar\phi_X(X_i)- \frac{1}{n}\sum_k\bar\phi_X(X_k)
\end{align*}
we can therefore expand $\tilde{K}$ in terms of $\bar{K}$ as
\begin{align*}
&\tilde{K}_{ij} \\ 
&= \langle\phi_X(X_i)- \frac{1}{n}\sum_k\phi_X(X_k),\phi_X(X_j) - \frac{1}{n}\sum_k\phi_X(X_k)\rangle \\
&= \langle\bar\phi_X(X_i)- \frac{1}{n}\sum_k\bar\phi_X(X_k),\bar\phi_X(X_j) - \frac{1}{n}\sum_k\bar\phi_X(X_k)\rangle \\
&= \bar{K}_{ij} - \frac{1}{n}\sum_k\bar{K}_{ik} - \frac{1}{n}\sum_k\bar{K}_{jk} + \frac{1}{n^2}\sum_{kl}\bar{K}_{kl}
\end{align*}

and expanding $\tilde{L}$ and $\tilde{M}$ in a similar way, we can rewrite the Lancaster test statistic as
\begin{align*} 
n\|\hat \mu_L\|^2 &= 
\frac{1}{n}(\bar{K} \circ \bar{L}\circ \bar{M})_{++} &&-
\frac{2}{n^2}((\bar{K}\circ \bar{L}) \bar{M})_{++} \\&- 
\frac{2}{n^2}((\bar{K} \circ \bar{M}) \bar{L})_{++} &&- 
\frac{2}{n^2}((\bar{M} \circ \bar{L}) \bar{K})_{++} \\&+ 
\frac{1}{n^3}(\bar{K} \circ \bar{L})_{++} \bar{M}_{++} &&+ 
\frac{1}{n^3}(\bar{K} \circ \bar{M})_{++} \bar{L}_{++} \\&+ 
\frac{1}{n^3}(\bar{L} \circ \bar{M})_{++} \bar{K}_{++} &&+ 
\frac{2}{n^3}(\bar{M}\bar{K}\bar{L})_{++} \\&+ 
\frac{2}{n^3}(\bar{K}\bar{L}\bar{M})_{++} &&+ 
\frac{2}{n^3}(\bar{K}\bar{M}\bar{L})_{++} \\&+ 
\frac{4}{n^3}tr(\bar{K}_+ \circ \bar{L}_+ \circ \bar{M}_+) &&-
\frac{4}{n^4}(\bar{K} \bar{L})_{++} \bar{M}_{++} \\& - 
\frac{4}{n^4}(\bar{K}\bar{M})_{++}\bar{L}_{++} &&- 
\frac{4}{n^4}(\bar{L}\bar{M})_{++} \bar{K}_{++} \\&+
\frac{4}{n^5}\bar{K}_{++} \bar{L}_{++} \bar{M}_{++} \\
\end{align*}
We denote by $C_{XYZ} = \mathbb{E}_{XYZ}[\bar\phi_X(X)\otimes\bar\phi_Y(Y)\otimes\bar\phi_Z(Z)]$ the population centred covariance operator with empirical estimate $\bar{C}_{XYZ} = \frac{1}{n}\sum_i\bar\phi_X(X_i)\otimes\bar\phi_Y(Y_i)\otimes\bar\phi_Z(Z_i)$. We define similarly the quantities $C_{XY}, C_{YZX}, \ldots$ with corresponding empirical counterparts $\bar{C}_{XY}, \bar{C}_{YZX}, \ldots$ where for example $C_{YZ} = \mathbb{E}_{YZ}[\bar\phi_Y(Y)\otimes\bar\phi_Z(Z)]$

Each of the terms in the above expression for $n\|\hat \mu_L\|^2$ can be expressed as inner products between empirical estimates of population centred covariance operators and tensor products of mean embeddings. Rewriting them as such yields:
\begin{align*}
n\|\hat \mu_L\|^2 &= n\langle \bar{C}_{XYZ},\bar{C}_{XYZ} \rangle \\& -
2n\langle \bar{C}_{XYZ},\bar{C}_{XY}\otimes\bar{\mu}_Z \rangle \\& -
2n\langle \bar{C}_{XZY},\bar{C}_{XZ}\otimes\bar{\mu}_Y \rangle \\& -
2n\langle \bar{C}_{YZX},\bar{C}_{YZ}\otimes\bar{\mu}_X \rangle \\& +
n\langle \bar{C}_{XY}\otimes\bar{\mu}_Z,\bar{C}_{XY}\otimes\bar{\mu}_Z \rangle \\& +
n\langle \bar{C}_{XZ}\otimes\bar{\mu}_Y,\bar{C}_{XZ}\otimes\bar{\mu}_Y \rangle \\& +
n\langle \bar{C}_{YZ}\otimes\bar{\mu}_X,\bar{C}_{YZ}\otimes\bar{\mu}_X \rangle \\& +
2n\langle \bar{\mu}_Z\otimes\bar{C}_{XY},\bar{C}_{ZX}\otimes\bar{\mu}_Y \rangle \\
& \enspace \vdots
\end{align*}
\begin{align*}
&+2n\langle \bar{\mu}_X\otimes\bar{C}_{YZ},\bar{C}_{XY}\otimes\bar{\mu}_Z \rangle \\& +
2n\langle \bar{\mu}_X\otimes\bar{C}_{ZY},\bar{C}_{XZ}\otimes\bar{\mu}_Y \rangle \\& +
4n\langle \bar{C}_{XYZ},\bar{\mu}_X \otimes\bar{\mu}_Y \otimes \bar{\mu}_Z \rangle \\& -
4n\langle \bar{C}_{XY}\otimes \bar{\mu}_Z,\bar{\mu}_X \otimes\bar{\mu}_Y \otimes \bar{\mu}_Z \rangle \\& -
4n\langle \bar{C}_{XZ}\otimes \bar{\mu}_Y,\bar{\mu}_X \otimes\bar{\mu}_Z \otimes \bar{\mu}_Y \rangle \\& -
4n\langle \bar{C}_{YZ}\otimes \bar{\mu}_X,\bar{\mu}_Y \otimes\bar{\mu}_Z \otimes \bar{\mu}_X \rangle \\& +
4n\langle \bar{\mu}_X \otimes\bar{\mu}_Y \otimes \bar{\mu}_Z,\bar{\mu}_X \otimes\bar{\mu}_Y \otimes \bar{\mu}_Z \rangle \\
\end{align*}

\vspace{-0.5cm}
By assumption, $\mathbb{P}_{XYZ} =\mathbb{P}_{XY}\mathbb{P}_{Z}$ and thus the expectation operator also factorises similarly. As a consequence, $C_{XYZ}=0$. Indeed, given any $A \in \mathcal{F}_X \otimes \mathcal{F_Y} \otimes \mathcal{F_Z}$, we can consider $A$ to be a bounded linear operator $\mathcal{F_Z} \longrightarrow\mathcal{F}_X \otimes \mathcal{F_Y} $. It follows that\footnote{We can bring the $\mathbb{E}_Z$ inside the inner product in the penultimate line due to the Bochner integrability of $\bar{\phi}_Z(Z)$, which follows from the conditions required for $\mu_Z$ to exist \citep{steinwart2008support}. }
\begin{align*}
&\mathbb{E}_{XYZ}\langle A, \bar{C}_{XYZ}\rangle \\
&= \frac{1}{n}\sum_i \mathbb{E}_{XY}\mathbb{E}_{Z}\langle A, \bar{\phi}_X(X_i)\otimes \bar{\phi}_Y(Y_i) \otimes \bar{\phi}_Z(Z_i) \rangle\\
&= \frac{1}{n}\sum_i \mathbb{E}_{XY}\mathbb{E}_{Z}\langle\bar{\phi}_X(X_i)\otimes \bar{\phi}_Y(Y_i),  A  \bar{\phi}_Z(Z_i) \rangle_{\mathcal{F}_{X}\otimes \mathcal{F}_Y}\\
&= \frac{1}{n}\sum_i \mathbb{E}_{XY}\langle\bar{\phi}_X(X_i)\otimes \bar{\phi}_Y(Y_i),  A \mathbb{E}_{Z} \bar{\phi}_Z(Z_i) \rangle_{\mathcal{F}_{X}\otimes \mathcal{F}_Y}\\
& = 0\\
\end{align*}

\vspace{-0.9cm}
We conclude that $C_{XYZ} = \mathbb{E}_{XYZ} \bar{C}_{XYZ} = 0$.

Similarly, $C_{XZY}$, $C_{YZX}$, $C_{XZ}$, $C_{YZ}$ are all 0 in their respective Hilbert spaces. Lemma \ref{lemma:beta} tells us that each subprocess of $(X_i,Y_i,Z_i)$ satisfies the same $\beta$-mixing conditions as $(X_i,Y_i,Z_i)$, thus by applying Lemma \ref{lemma:hilbertCLT} it follows that $\|\bar{C}_{XZY}\|$, $\|\bar{C}_{YZX}\|$, $\|\bar{C}_{XZ}\|$, $\|\bar{C}_{YZ}\|$, $\|\bar{\mu}_X\|$, $\|\bar{\mu}_Y\|$, $\|\bar{\mu}_Z\| = O_P\left(n^{-\frac{1}{2}}\right)$. Therefore
\begin{align*}
n\|&\hat \mu_L\|^2  \xrightarrow{O_P(n^{-\frac{1}{2}})} n\langle \bar{C}_{XYZ},\bar{C}_{XYZ} \rangle \\ &-
2n\langle \bar{C}_{XYZ},\bar{C}_{XY}\otimes\bar{\mu}_Z \rangle -
2n\langle \bar{C}_{XZY},\bar{C}_{XZ}\otimes\bar{\mu}_Y \rangle \\ &=
\frac{1}{n}((\bar{K}\circ \bar{L}) \circ \bar{M})_{++}\\& - \frac{2}{n^2}((\bar{K}\circ \bar{L})\bar{M})_{++} + \frac{1}{n^3}(\bar{K}\circ \bar{L})_{++}\bar{M}_{++}
\end{align*}

since all the other terms decay at least as quickly as $O_P(\frac{1}{\sqrt{n}})$. This is shown here for $n\langle \bar{\mu}_X\otimes\bar{C}_{YZ},\bar{C}_{XY}\otimes\bar{\mu}_Z \rangle$; the proofs for the other terms are similar.
\begin{align*}
&n\langle \bar{\mu}_X\otimes\bar{C}_{YZ},\bar{C}_{XY}\otimes\bar{\mu}_Z \rangle \\
&\leq n \| \bar{\mu}_X\otimes\bar{C}_{YZ}\| \|\bar{C}_{XY}\otimes\bar{\mu}_Z \| \\
& = n\sqrt{\langle \bar{\mu}_X\otimes\bar{C}_{YZ} , \bar{\mu}_X\otimes\bar{C}_{YZ} \rangle} \sqrt{\langle \bar{C}_{XY}\otimes\bar{\mu}_Z, \bar{C}_{XY}\otimes\bar{\mu}_Z \rangle} \\
\end{align*}
\begin{align*}
& = n\sqrt{\langle \bar{\mu}_X, \bar{\mu}_X \rangle \langle \bar{C}_{YZ} , \bar{C}_{YZ} \rangle} \sqrt{\langle \bar{C}_{XY}, \bar{C}_{XY} \rangle \langle \bar{\mu}_Z, \bar{\mu}_Z \rangle} \\
& =  n \| \bar{\mu}_X\|\|\bar{C}_{YZ}\| \|\bar{C}_{XY}\|\|\bar{\mu}_Z \| \\
& = n \mathsmaller{O_P\left(\frac{1}{\sqrt{n}}\right)} \mathsmaller{O_P\left(\frac{1}{\sqrt{n}}\right)} O_P(1) \mathsmaller{O_P\left(\frac{1}{\sqrt{n}}\right)} = \mathsmaller{O_P\left(\frac{1}{\sqrt{n}}\right)}
\end{align*}

It can be shown that $\bar{K}\circ \bar{L}$ in the above expression can be replaced with $\overline{\bar{K}\circ \bar{L}}$ while preserving equality. That is, we can equivalently write
\begin{align*}
n\|\Delta_L \hat{P}\|^2 & \longrightarrow \frac{1}{n}((\overline{\bar{K}\circ \bar{L}}) \circ \bar{M})_{++}\\& - \frac{2}{n^2}((\overline{\bar{K}\circ \bar{L}})\bar{M})_{++} + \frac{1}{n^3}(\overline{\bar{K}\circ \bar{L}})_{++}\bar{M}_{++}
\end{align*}
This is equivalent to treating $\bar{k}\otimes\bar{l}$ as a kernel on the single variable $T:=(X,Y)$ and performing another recentering trick as we did at the beginning of this proof. By rewriting the above expression in terms of the operator $\bar{C}_{TZ}$ and mean embeddings $\mu_T$ and $\mu_Z$, it can be shown by a similar argument to before that the latter two terms tend to 0 at least as $O_P(n^{-\frac{1}{2}})$, and thus, substituting for the definition of $\|\hat \mu^{(Z)}_{L,2} \|^2$,
\begin{align*}
 n \|\hat \mu_{L} \|^2 \xrightarrow{O_P(\frac{1}{\sqrt{n}})} n \|\hat \mu^{(Z)}_{L,2} \|^2
\end{align*} as required.
\end{proof}

\begin{proof}(Theorem \ref{theorem:degenerate-kernel})

Note that $\mathbb{E}_{XYZ} = \mathbb{E}_{XY}\mathbb{E}_Z$ under $\mathcal{H}_Z$. Therefore, fixing any $s_j = (x_j,y_j,z_j)$ we have that
\begin{align*}
\mathbb{E}_{S_i}&h(S_i,s_j) = \mathbb{E}_{X_iY_i} \mathbb{E}_{Z_i}\overline{\bar{k}\otimes\bar{l}}\otimes\bar{m} (S_i,s_j) \\
 &=  \langle\mathbb{E}_{X_iY_i}\bar{\phi}(X_i)\otimes\bar{\phi}(Y_i) - C_{XY},\bar{\phi}(x_j)\otimes\bar{\phi}(y_j) - C_{XY}\rangle \\
 &\quad \quad \quad \quad \times \langle \mathbb{E}_{Z_i}\bar{\phi}(Z_i),\bar{\phi}(z_j)\rangle \\
 &=  \langle 0 ,\bar{\phi}(x_j)\otimes\bar{\phi}(y_j) - C_{XY}\rangle \\
 &\quad \quad \quad \quad \times \langle 0 ,\bar{\phi}(z_j)\rangle  = 0
\end{align*}
Therefore $h$ is degenerate. Symmetry follows from the symmetry of the Hilbert space inner product.

For boundedness and Lipschitz continuity, it suffices to show the two following rules for constructing new kernels from old preserve both properties (see Supplementary materials \ref{supp:bounded-and-lipschitz} for proof):
\begin{itemize} \setlength\itemsep{0em}
\item $k \mapsto \bar{k}$ 
\item $(k,l) \mapsto k \otimes l$
\end{itemize}
It then follows that $h = \overline{\bar{k}\otimes\bar{l}}\otimes\bar{m}$ is bounded and Lipschitz continuous since it can be constructed from $k$, $l$ and $m$ using the two above rules.
\end{proof}

\newpage
\section*{References}
\bibliography{ref}

\begin{thebibliography}{29}
\providecommand{\natexlab}[1]{#1}
\providecommand{\url}[1]{\texttt{#1}}
\expandafter\ifx\csname urlstyle\endcsname\relax
  \providecommand{\doi}[1]{doi: #1}\else
  \providecommand{\doi}{doi: \begingroup \urlstyle{rm}\Url}\fi

\bibitem[Besserve et~al.(2013)Besserve, Logothetis, and
  Sch{\"o}lkopf]{besserve_statistical_2013}
M.~Besserve, N.~Logothetis, and B.~Sch{\"o}lkopf.
\newblock Statistical analysis of coupled time series with kernel
  cross-spectral density operators.
\newblock In \emph{NIPS}, pages 2535--2543, 2013.

\bibitem[Bradley et~al.(2005)]{bradley2005basic}
R.~C. Bradley et~al.
\newblock Basic properties of strong mixing conditions. a survey and some open
  questions.
\newblock \emph{Probability surveys}, 2\penalty0 (2):\penalty0 107--144, 2005.

\bibitem[Chwialkowski and Gretton(2014)]{chwialkowski2014kernel}
K.~Chwialkowski and A.~Gretton.
\newblock A kernel independence test for random processes.
\newblock \emph{arXiv preprint arXiv:1402.4501}, 2014.

\bibitem[Chwialkowski et~al.(2014)Chwialkowski, Sejdinovic, and
  Gretton]{chwialkowski2014wild}
K.~P. Chwialkowski, D.~Sejdinovic, and A.~Gretton.
\newblock A wild bootstrap for degenerate kernel tests.
\newblock In \emph{Advances in neural information processing systems}, pages
  3608--3616, 2014.

\bibitem[Dedecker et~al.(2007)Dedecker, Doukhan, Lang, Rafael, Louhichi, and
  Prieur]{dedecker2007weak}
J.~Dedecker, P.~Doukhan, G.~Lang, L.~R.~J. Rafael, S.~Louhichi, and C.~Prieur.
\newblock Weak dependence.
\newblock In \emph{Weak Dependence: With Examples and Applications}, pages
  9--20. Springer, 2007.

\bibitem[Dehling et~al.(2015)Dehling, Sharipov, and
  Wendler]{dehling2015bootstrap}
H.~Dehling, O.~S. Sharipov, and M.~Wendler.
\newblock Bootstrap for dependent hilbert space-valued random variables with
  application to von mises statistics.
\newblock \emph{Journal of Multivariate Analysis}, 133:\penalty0 200--215,
  2015.

\bibitem[Doukhan(1994)]{doukhan1994mixing}
P.~Doukhan.
\newblock \emph{Mixing}.
\newblock Springer, 1994.

\bibitem[Feuerverger(1993)]{Feuerverger93}
A.~Feuerverger.
\newblock A consistent test for bivariate dependence.
\newblock \emph{International Statistical Review}, 61\penalty0 (3):\penalty0
  419--433, 1993.

\bibitem[Fukumizu et~al.(2008)Fukumizu, Gretton, Sun, and
  Sch\"olkopf]{FukGreSunSch08}
K.~Fukumizu, A.~Gretton, X.~Sun, and B.~Sch\"olkopf.
\newblock Kernel measures of conditional dependence.
\newblock In \emph{NIPS}, pages 489--496, Cambridge, MA, 2008. MIT Press.

\bibitem[Gaisser et~al.(2010)Gaisser, Ruppert, and Schmid]{GaiRupSch10}
S.~Gaisser, M.~Ruppert, and F.~Schmid.
\newblock A multivariate version of hoeffding’s phi-square.
\newblock \emph{Journal of Multivariate Analysis}, 101\penalty0 (10):\penalty0
  2571--2586, 2010.

\bibitem[Gretton and Gyorfi(2010)]{GreGyo10}
A.~Gretton and L.~Gyorfi.
\newblock Consistent nonparametric tests of independence.
\newblock \emph{Journal of Machine Learning Research}, 11:\penalty0 1391--1423,
  2010.

\bibitem[Gretton et~al.(2005)Gretton, Bousquet, Smola, and
  Sch{\"o}lkopf]{gretton2005measuring}
A.~Gretton, O.~Bousquet, A.~Smola, and B.~Sch{\"o}lkopf.
\newblock Measuring statistical dependence with hilbert-schmidt norms.
\newblock In \emph{Algorithmic learning theory}, pages 63--77. Springer, 2005.

\bibitem[Gretton et~al.(2007)Gretton, Fukumizu, Teo, Song, Sch{\"o}lkopf, and
  Smola]{gretton2007kernel}
A.~Gretton, K.~Fukumizu, C.~H. Teo, L.~Song, B.~Sch{\"o}lkopf, and A.~J. Smola.
\newblock A kernel statistical test of independence.
\newblock In \emph{Advances in Neural Information Processing Systems}, pages
  585--592, 2007.

\bibitem[Heller et~al.(2013)Heller, Heller, and Gorfine]{HelHelGor13}
R.~Heller, Y.~Heller, and M.~Gorfine.
\newblock A consistent multivariate test of association based on ranks of
  distances.
\newblock \emph{Biometrika}, 100\penalty0 (2):\penalty0 503--510, 2013.

\bibitem[Holm(1979)]{holm1979simple}
S.~Holm.
\newblock A simple sequentially rejective multiple test procedure.
\newblock \emph{Scandinavian journal of statistics}, pages 65--70, 1979.

\bibitem[Kankainen and Ushakov(1998)]{KanUsh98}
A.~Kankainen and N.~G. Ushakov.
\newblock A consistent modification of a test for independence based on the
  empirical characteristic function.
\newblock \emph{Journal of Mathematical Sciencies}, 89:\penalty0 1582--1589,
  1998.

\bibitem[Kirchg{\"a}ssner et~al.(2012)Kirchg{\"a}ssner, Wolters, and
  Hassler]{kirchgassner2012introduction}
G.~Kirchg{\"a}ssner, J.~Wolters, and U.~Hassler.
\newblock \emph{Introduction to modern time series analysis}.
\newblock Springer Science \& Business Media, 2012.

\bibitem[Lancaster(1969)]{lancaster1969chi}
H.~O. Lancaster.
\newblock \emph{Chi-Square Distribution}.
\newblock Wiley Online Library, 1969.

\bibitem[Ledford and Tawn(1996)]{ledford1996statistics}
A.~W. Ledford and J.~A. Tawn.
\newblock Statistics for near independence in multivariate extreme values.
\newblock \emph{Biometrika}, 83\penalty0 (1):\penalty0 169--187, 1996.

\bibitem[Leucht and Neumann(2013)]{leucht2013dependent}
A.~Leucht and M.~H. Neumann.
\newblock Dependent wild bootstrap for degenerate u-and v-statistics.
\newblock \emph{Journal of Multivariate Analysis}, 117:\penalty0 257--280,
  2013.

\bibitem[Patra et~al.(2015)Patra, Sen, and Szekely]{PatSenSze15}
R.~Patra, B.~Sen, and G.~Szekely.
\newblock On a nonparametric notion of residual and its applications.
\newblock \emph{Statist. Probab. Lett.}, 106:\penalty0 208--213, 2015.

\bibitem[Sejdinovic et~al.(2013)Sejdinovic, Gretton, and
  Bergsma]{sejdinovic2013kernel}
D.~Sejdinovic, A.~Gretton, and W.~Bergsma.
\newblock A kernel test for three-variable interactions.
\newblock In \emph{Advances in Neural Information Processing Systems}, pages
  1124--1132, 2013.

\bibitem[Serfling(2009)]{serfling2009approximation}
R.~J. Serfling.
\newblock \emph{Approximation theorems of mathematical statistics}, volume 162.
\newblock John Wiley \&amp; Sons, 2009.

\bibitem[Shao(2010)]{shao2010dependent}
X.~Shao.
\newblock The dependent wild bootstrap.
\newblock \emph{Journal of the American Statistical Association}, 105\penalty0
  (489):\penalty0 218--235, 2010.

\bibitem[Smola et~al.(2007)Smola, Gretton, Song, and
  Sch{\"o}lkopf]{smola2007hilbert}
A.~Smola, A.~Gretton, L.~Song, and B.~Sch{\"o}lkopf.
\newblock A hilbert space embedding for distributions.
\newblock In \emph{Algorithmic Learning Theory}, pages 13--31. Springer, 2007.

\bibitem[Sriperumbudur et~al.(2011)Sriperumbudur, Fukumizu, and
  Lanckriet]{sriperumbudur2011universality}
B.~K. Sriperumbudur, K.~Fukumizu, and G.~R. Lanckriet.
\newblock Universality, characteristic kernels and rkhs embedding of measures.
\newblock \emph{The Journal of Machine Learning Research}, 12:\penalty0
  2389--2410, 2011.

\bibitem[Steinwart and Christmann(2008)]{steinwart2008support}
I.~Steinwart and A.~Christmann.
\newblock \emph{Support vector machines}.
\newblock Springer Science \&amp; Business Media, 2008.

\bibitem[Sz\'{e}kely et~al.(2007)Sz\'{e}kely, Rizzo, and Bakirov]{Szekely2007}
G.~Sz\'{e}kely, M.~Rizzo, and N.~Bakirov.
\newblock Measuring and testing dependence by correlation of distances.
\newblock \emph{Annals of Statistics}, 35\penalty0 (6):\penalty0 2769--2794,
  2007.

\bibitem[Zhang et~al.(2011)Zhang, Peters, Janzing, and Schoelkopf]{Zhang2011}
K.~Zhang, J.~Peters, D.~Janzing, and B.~Schoelkopf.
\newblock Kernel-based conditional independence test and application in causal
  discovery.
\newblock In \emph{Proceedings of the Conference on Uncertainty in Artificial
  Intelligence (UAI)}, pages 804--813, 2011.

\end{thebibliography}

%%%%%%%%%%%%%%%%%%%%%%%%%%%%%%%%%%%%%%%%%%%
%%%%%%%%%%%%%%%%%%%%%%%%%%%%%%%%%%%%%%%%%%%
\appendix
\onecolumn

\section{SUPPLEMENTARY MATERIAL}
This supplementary section contains proofs omitted from the main paper and includes a proof that the HSIC statistic asymptotically satisfies the hypothesis of the Wild Bootstrap.

\subsection{HILBERT SPACE RANDOM VARIABLE CLT}\label{supp:hilbert-clt}

In this paper we exploit a Central Limit Theorem for Hilbert space valued random variables that are functions of random processes \citep{dehling2015bootstrap}. One of the conditions required to apply this theorem concerns appropriate $\beta$-mixing of the underlying processes. This theorem is used as a black-box, and it is hoped by the authors that as further theorems concerning CLT-properties of Hilbert space random variables are developed, the conditions required of the processes may be weakened.
\begin{proof}(Lemma \ref{lemma:hilbertCLT})
We exploit Theorem 1.1 from \citet{dehling2015bootstrap}. Using the language of this paper, $\bar{\phi}(X_i)$ is a 1-approximating functional of $(X_i)_i$, following straightforwardly from the definition of 1-approximating functionals given. 

Since our kernels are bounded, $\exists C: \enspace \|\bar{\phi}(X_i)\| < C $ and so \[\mathbb{E}\|\bar{\phi}(X_1)\|^{2+\delta} <C^{2+\delta}< \infty \enspace \forall \delta>0\]
Thus condition (1) is satisfied.

We can take $f_m = \bar{\phi}(X_0)\enspace \forall m$ and so achieve $a_m= 0 \enspace \forall m$, thus condition (2) is satisfied.

By assumption on the time series, condition (3) is satisfied.

Thus, by Theorem 1.1 in \citet{dehling2015bootstrap}
\[\sqrt{n} (\tilde{\mu}_{X} - \mu_{X}) \widesim[2]{n\longrightarrow\infty} N\]
where $N$ is a Hilbert space valued Gaussian random variable and convergence is in distribution. Thus 
\[\|\tilde{\mu}_{X} - \mu_{X}\| = O_P(\frac{1}{\sqrt{n}})\]
\end{proof}
%%%

\subsection{SUB-PROCESSES OF $\beta$-MIXING PROCESSES ARE $\beta$-MIXING}\label{supp:lemma-beta}
\begin{lemma}\label{lemma:beta}
Suppose that the process $(X_t,Y_t,Z_t)_t$ is $\beta$-mixing. Then any `sub-process' is also $\beta$-mixing (for example $(X_t,Y_t)_t$ or $(X_t)_t$)
\end{lemma}
\begin{proof}(Lemma \ref{lemma:beta}) 

Let us consider $(X_t,Y_t)_t$.
Let us call $\beta_{XYZ}(m)$ the coefficients for the process $(X_t,Y_t,Z_t)_t$, and $\beta_{XY}(m)$ the coefficients for the process $(X_t,Y_t)_t$. 

Observe that for $A \in \sigma((X_b,Y_b),\ldots, (X_c,Y_c))$, it is the case that $A \times \mathcal{Z} \in \sigma((X_b,Y_b,Z_b),\ldots, (X_c,Y_c,Z_c))$ and $\mathbb{P}_{XY}(A) = \mathbb{P}_{XYZ}(A\times \mathcal{Z})$.

Thus

\begin{align*}
\beta_{XY}(m) &= \frac{1}{2} \sup_n \sup_{ \{A_i^{XY} \}, \{B_j^{XY} \} } \sum_{i=1}^I \sum_{j=1}^J | \mathbb{P}_{XY}(A_i^{XY} \cap B_j^{XY}) - \mathbb{P}_{XYZ}(A_i^{XY})\mathbb{P}_{XYZ}(B_j^{XY})| \\
&= \frac{1}{2} \sup_n \sup_{ \{A_i^{XY} \}, \{B_j^{XY} \} } \sum_{i=1}^I \sum_{j=1}^J | \mathbb{P}_{XYZ}((A_i^{XY}\times \mathcal{Z}) \cap (B_j^{XY} \times \mathcal{Z})) \\& \quad \quad\quad \quad \quad \quad\quad \quad \quad \quad\quad \quad- \mathbb{P}_{XYZ}(A_i^{XY}\times \mathcal{Z})\mathbb{P}_{XYZ}(B_j^{XY} \times \mathcal{Z})| \\
& \leq \frac{1}{2} \sup_n \sup_{ \{A_i^{XYZ} \}, \{B_j^{XYZ} \} } \sum_{i=1}^I \sum_{j=1}^J | \mathbb{P}_{XYZ}(A_i^{XYZ} \cap B_j^{XYZ}) - \mathbb{P}_{XYZ}(A_i^{XYZ})\mathbb{P}_{XYZ}(B_j^{XYZ})| \\
& = \beta_{XYZ}(m)
\end{align*}

Thus we have shown that  $\beta_{XYZ}(m) \longrightarrow 0 \implies \beta_{XY}(m) \longrightarrow 0$. That is, if  $(X_t,Y_t,Z_t)_t$ is $\beta$-mixing then so is  $(X_t,Y_t)_t$ 

A similar argument holds for any other sub-process.
\end{proof}

\subsection{CONTROL OF TYPE I ERROR}\label{supp:quantile-proof}
Theorem \ref{theorem:quantiles-converge} shows that the quantiles of the bootstrapped statistic $nV_b$ (which we can estimate by drawing a large number of samples) converge to those of the test statistic $\|\hat \mu_L\|^2$ under the null hypothesis. Therefore, we can estimate rejection thresholds to appropriately control Type I error.

\begin{proof}(Theorem \ref{theorem:quantiles-converge})

We use Theorem 3.1 from \citet{leucht2013dependent}. By assumption, condition (B2) is satisfied by the random matrix $W$. (A2) is satisfied due to Theorem \ref{theorem:degenerate-kernel}. (B1) is satisfied due to the suitable mixing assumptions.

Therefore, Theorem 3.1 implies that $nV_b$ converges in probability to the null distribution of $n\|\hat \mu_{L,2}^{(Z)}\|^2$. Since $n\|\mu_L\|^2$ also converges in probability to $n\|\hat \mu_{L,2}^{(Z)}\|^2$, it follows that $nV_b$ converges to $n\|\mu_L\|^2$ in probability, and thus also in distribution. Convergence in distribution implies that the quantiles converge.
\end{proof}

\subsection{SEMI-CONSISTENCY}\label{supp:consistent}
Theorem \ref{theorem:consistent} provides a consistency result: if $\Delta_LP\not = 0$, then we correctly reject $\mathcal{H}_0$ with probability 1 in the limit $n\longrightarrow\infty$.

\begin{proof}
By Theorem 2 from \citet{chwialkowski2014wild}, $nV_b$ converges to some random variable with finite variance, while $n\|\hat{\mu}_L\|^2 \longrightarrow \infty$. Thus if $Q_\alpha$ is the $\alpha$-quantile of $nV_b$, then $P(n\|\hat{\mu}_L\|^2 > Q_\alpha) \longrightarrow 1 $ for any $\alpha$.
\end{proof}

\subsection{PROOF THAT BOUNDEDNESS AND LIPSCHITZ CONTINUITY IS PRESERVED}\label{supp:bounded-and-lipschitz}
Recall that a kernel $k$ defined on $\mathcal{X}$ is Lipschitz continuous iff $\exists C_k : \forall w \enspace |k(x,w) - k(x',w)| \leq C_k d_\mathcal{X}(x,x')$ where $d_\mathcal{X}$ is the metric on $\mathcal{X}$ with respect to which $k$ is Lipschitz continuous.

\begin{claim}
$k$ bounded and Lipschitz continuous $\implies$ $\bar{k}$ is bounded and Lipschitz continuous
\end{claim}
\begin{proof}
$k$ bounded implies there exists $B_k$ such that $|k(x,w)|\leq B_k$ $\forall x, w \in \mathcal{X}$. It follows that

\begin{align*}
|\bar{k}(x,w)| & = |k(x,w) - \mathbb{E}_X[k(X,w)] - \mathbb{E}_W[k(x,W)] + \mathbb{E}_{XW}[k(X,W)]| \\
& \leq |k(x,w)|  + \mathbb{E}_X |k(X,w)| + \mathbb{E}_W|k(x,W)| + \mathbb{E}_{XW}|k(X,W)| \\
& \leq 4B_k \\
\end{align*}

And thus $\bar{k}$ is bounded. For Lipschitz continuity, observe that for any $w \in \mathcal{X}$

\begin{align*}
|\bar{k}(x,w) - \bar{k}(x',w)| & = |k(x,w) - \mathbb{E}_X[k(X,w)] - \mathbb{E}_W[k(x,W)] + \mathbb{E}_{XW}[k(X,W)] \\
& \quad \quad - k(x',w) + \mathbb{E}_X[k(X,w)] + \mathbb{E}_W[k(x',W)] - \mathbb{E}_{XW}[k(X,W)] |\\
& = |k(x,w)- k(x',w)  + \mathbb{E}_W[k(x',W)] - \mathbb{E}_W[k(x,W)] |\\
&\leq |k(x,w)- k(x',w)| + |\mathbb{E}_W[k(x',W)] - \mathbb{E}_W[k(x,W)] | \\
&\leq|k(x,w)- k(x',w)| + \mathbb{E}_W|k(x',W) - k(x,W) | \\
&\leq 2C_k d_\mathcal{X}(x,x')
\end{align*}

and thus $\bar{k}$ is Lipschitz continuous.

\end{proof}

\begin{claim}
$k$ and $l$ bounded and Lipschitz continuous with respect to the metrics $d_\mathcal{X}$ and $d_\mathcal{Y}$ respectively $\implies$ $k\otimes l$ is bounded and Lipschitz continuous with respect to any metric on $\mathcal{X}\times \mathcal{Y}$ equivalent to $d \left( (x,y),(x',y') \right) =  d_\mathcal{X}(x,x') + d_\mathcal{Y}(y,y')$
\end{claim}

Note that all norms on finite dimensional vector spaces are equivalent, and so if $\mathcal{X}$ and $\mathcal{Y}$ are finite dimensional vector spaces then $k\otimes l$ is Lipschitz continuous with respect to \emph{any} norm on $\mathcal{X}\times \mathcal{Y}$

\begin{proof} Let $k$ and $l$ be bounded by $B_k$ and $B_l$ respectively. Then 

\begin{align*}
|k \otimes l \left( (x,y), (w,z) \right)| &= |k(x,w)l(y,z)| \\
&= |k(x,w)||l(y,z)| \\
&\leq B_k B_l \\
\end{align*}

Let $k$ and $l$ have Lipschitz constants $C_k$ and $C_l$ respectively. Then, for any $(w,z) \in \mathcal{X\times Y}$

\begin{align*}
|k \otimes l \left( (x,y), (w,z) \right) &- k \otimes l \left( (x',y'), (w,z) \right) |  \\
& = |k(x,w)l(y,z) - k(x',w)l(y',z)| \\
& = |k(x,w)l(y,z) - k(x',w)l(y,z) + k(x',w)l(y,z) - k(x',w)l(y',z)| \\
& \leq |l(y,z)| |k(x,w) - k(x',w)| + |k(x',w)||l(y,z) - l(y',z)| \\
& \leq B_l C_k d_\mathcal{X}(x,x') + B_k C_l d_\mathcal{Y}(y,y') \\
& \leq \max(B_l C_k, B_k C_l )  \enspace d\left((x,y),(x',y')\right) 
\end{align*}

\end{proof}

\subsection{PROOF THAT HSIC CAN BE WILD BOOTSTRAPPED}

Given samples $\{(X_i,Y_i)\}_{i=1}^n$, and taking all notation involving kernels and base spaces as before, the HSIC statistic is defined to be the squared RKHS distance between the empirical embeddings of the distributions $\mathbb{P}_{XY}$ and $\mathbb{P}_X\mathbb{P}_Y$:

\begin{align*}
HSIC_b & = \| \frac{1}{n}\sum_i \phi_X(X_i) \otimes \phi_Y(Y_i) - \left(\frac{1}{n}\sum_i \phi_X(X_i)\right) \otimes \left(\frac{1}{n}\sum_i \phi_Y(Y_i) \right)\|^2 \\
& = \frac{1}{n^2} (K\circ L)_{++}  - \frac{2}{n^3}(KL)_{++} + \frac{1}{n^4}K_{++}L_{++} \\
& = \frac{1}{n^2}(\tilde{K}\circ \tilde{L})_{++}
\end{align*}
where the last equality can be shown easily by expanding $\tilde{K}$ (and $\tilde{L}$ similarly) as
\begin{align*}
\tilde{K}_{ij} &= \langle\phi_X(X_i)- \frac{1}{n}\sum_k\phi_X(X_k),\phi_X(X_j) - \frac{1}{n}\sum_k\phi_X(X_k)\rangle \\
&= K_{ij} - \frac{1}{n}\sum_kK_{ik} - \frac{1}{n}\sum_kK_{jk} + \frac{1}{n^2}\sum_{kl}K_{kl}
\end{align*} 

\begin{theorem}\label{theorem:HSIC-conv-in-prob} Suppose that $(X_i,Y_i)_{i=1}^n$ are drawn from a process that is $\beta$-mixing with coefficients $\beta(m)$ satisfying $\sum_{m=1}^{\infty}\beta(m)^{\frac{\delta}{2+\delta}}<\infty$ for some $\delta>0$. Under $\mathcal{H}_0 = \{ \mathbb{P}_{XY} = \mathbb{P}_X\mathbb{P}_Y\}$, $\lim_{n \to \infty} ( nHSIC_b - \frac{1}{n} (\bar{K}\circ \bar{L})_{++} ) =0 $ in probability.
\end{theorem}
Similar to the case with the Lancaster statistic, $\frac{1}{n} (\bar{K}\circ \bar{L})_{++}$ is much easier to study than $nHSIC_b$ under the non-\emph{i.i.d.}~assumption. It can be written as a normalised $V$-statistic as:
 
\[ 
nV_n = \frac{1}{n} \mathlarger{\sum}_{1\leq i,j \leq n} \bar{k} \otimes \bar{l}(S_i,S_j)
\]
where  $S_i = (X_i,Y_i)$. Again, the crucial observation is that
\[
 h = \bar{k} \otimes \bar{l}
\]
is well behaved in the following sense 
\begin{theorem}\label{theorem:HSIC-degenerate-kernel}
Suppose that $k$ and $l$ are bounded symmetric Lipschitz contentious kernels. Then $h$ is also bounded symmetric and Lipschitz continuous, which is moreover degenerate under $\mathcal{H}_0$.
\end{theorem}

Together, Theorems \ref{theorem:HSIC-conv-in-prob} and \ref{theorem:HSIC-degenerate-kernel} justify use of the Wild Bootstrap in estimating the quantiles of the null distribution of the test statistic $nHSIC_b$.

\begin{proof}(Theorem \ref{theorem:HSIC-conv-in-prob})
We can equivalently write $HSIC_b$ as the norm of the empirically centred covariance operator, which is invariant to population centering the feature maps:

\begin{align*}
HSIC_b &= \left\| \frac{1}{n}\sum_i\left( \phi_X(X_i) - \frac{1}{n}\sum_j \phi_X(X_j) \right) \otimes \frac{1}{n}\sum_i\left( \phi_Y(Y_i) - \frac{1}{n}\sum_j \phi_Y(Y_j) \right)\right\|^2 \\
& = \left\| \frac{1}{n}\sum_i\left( \bar\phi_X(X_i) - \frac{1}{n}\sum_j \bar\phi_X(X_j) \right) \otimes \frac{1}{n}\sum_i\left( \bar\phi_Y(Y_i) - \frac{1}{n}\sum_j \bar\phi_Y(Y_j) \right)\right\|^2 \\
\end{align*} 

Expanding this, we can rewrite the above in terms of inner products involving the population centred covariance operator and the population centred mean embeddings:

\begin{align*}
nHSIC_b &= n\| \bar{C}_{XY} \|^2 - 2n \langle \bar{C}_{XY} , \bar{\mu}_X \otimes \bar{\mu}_Y \rangle + n\| \bar{\mu}_X \otimes \bar{\mu}_Y  \|^2
\end{align*} 

The first term in this expression can be written as $n\| \bar{C}_{XY} \|^2 = \frac{1}{n}\sum_{ij}\bar{k}(X_i,X_j)\bar{l}(Y_i,Y_j)= \frac{1}{n}\sum_{ij} h(S_i,S_j)$. We show that the remaining two terms decay to zero in probability.

\pagebreak
By assumption, $\mathbb{P}_{XY} = \mathbb{P}_X\mathbb{P}_Y$ and thus the expectation operator factorises similarly. Therefore, for any $A\in HS(\mathcal{F}_Y,\mathcal{F}_X)$, 
\begin{align*}
\mathbb{E}_{XY} \langle A , \bar{C}_{XY} \rangle & = \frac{1}{n}\sum_i \mathbb{E}_{X}\mathbb{E}_{Y} \langle A , \left( \phi_X(X_i) - \mu_X \right) \otimes \left( \phi_Y(Y_i) - \mu_Y \right)\rangle_{HS} \\
& = \frac{1}{n}\sum_i \mathbb{E}_{X}\mathbb{E}_{Y}\langle \phi_X(X_i) - \mu_X  ,  A \left( \phi_Y(Y_i) - \mu_Y \right) \rangle_{\mathcal{F}_X} \\
& = \frac{1}{n}\sum_i \mathbb{E}_{Y}\langle \mathbb{E}_{X}\left(\phi_X(X_i) - \mu_X\right)  ,  A \left( \phi_Y(Y_i) - \mu_Y \right) \rangle_{\mathcal{F}_X} \\
& = 0
\end{align*}

where the commutativity of $\mathbb{E}_X$ with the inner product in the penultimate line follows from the Bochner integrability of the quantity $\phi_X(X) - \mu_X$, which in turn follows from the conditions under which $\mu_{X}$ exists \citep{steinwart2008support}. It follows that $\mathbb{E}_{XY} \bar{C}_{XY} = 0$.

Thus by Lemma \ref{lemma:hilbertCLT} as before, it follows that $\|\bar{C}_{XY}\|, \|\bar{\mu}_X\|, \|\bar{\mu}_Y\| = O_P(n^{-\frac{1}{2}})$. 

It thus follows that the two latter quantities in the above expression for $nHSIC_b$ decay to $0$ in probability.

\begin{align*}
n\langle \bar{C}_{XY},\bar{\mu}_X  \otimes \bar{\mu}_Y \rangle &\leq n \|\bar{C}_{XY}\|\|\bar{\mu}_X\| \| \bar{\mu}_Y\|\\
&= O_P(n^{-\frac{1}{2}})
\end{align*}

\begin{align*}
\| \bar{\mu}_X \otimes \bar{\mu}_Y \|^2 & = n\|\bar{\mu}_X \|^2 \|\bar{\mu}_Y \|^2\\ &= nO_P(n^{-2}) \\&= O_P(n^{-1})
\end{align*}

It follows that $nHSIC_b  \xrightarrow{O(n^{-\frac{1}{2}})} = n\|\bar{C}_{XY}\|^2 = \frac{1}{n} (\bar{K}\circ \bar{L})_{++}$, as required.
\end{proof}

\begin{proof}(Theorem \ref{theorem:HSIC-degenerate-kernel})

To show degeneracy, fix any $s_i$ and observe that 
\begin{align*}
\mathbb{E}_{S}h(s_i,S) &= \mathbb{E}_X\mathbb{E}_Y \langle\bar{\phi}(x_i),\bar{\phi}(X)\rangle \langle\bar{\phi}(y_i),\bar{\phi}(Y)\rangle \\
&= \langle\bar{\phi}(x_i),\mathbb{E}_X\bar{\phi}(X)\rangle \langle\bar{\phi}(y_i),\mathbb{E}_Y\bar{\phi}(Y)\rangle \\
&= \langle\bar{\phi}(x_i),0\rangle \langle\bar{\phi}(y_i),0\rangle = 0\\
\end{align*}	

Symmetry is inherited from symmetry of $k$ and $l$. Boundedness and Lipschitz continuity are implied by application of the claims in Section \ref{supp:bounded-and-lipschitz}.

\end{proof}

\end{document}